\documentclass[letterpaper, 10 pt, conference]{ieeeconf}
\IEEEoverridecommandlockouts{}   
\usepackage{cite}

\usepackage{amsmath,amssymb,amsfonts,mathtools,amsthm}

\usepackage{algorithmic}
\usepackage{graphicx}
\usepackage{textcomp}
\usepackage{color}
\usepackage{xcolor}

\newtheoremstyle{rmk}
  {0pt}{0pt}           % space above / below
  {\normalfont}        % body font: upright
  {\parindent}         % indent
  {\itshape}           % head font: italic
  {:}                  % punctuation after head
  {.5em}               % space after head
  {\itshape\thmname{#1}\thmnumber{ #2}\thmnote{ (#3)}}

\newcommand{\tcb}{\textcolor{black}}

\DeclareMathOperator*{\argmin}{arg\,min}

\newtheorem{thm}{Theorem}

\newtheorem{lemma}{Lemma}
\newtheorem{cor}{Corollary}
\newtheorem{defn}{Definition}
\newtheorem{assumption}{Assumption}

\theoremstyle{rmk}
\newtheorem{remark}{Remark}

\newcommand\pr{\text{Proj}}

\newcommand\gdec{\mathcal{G}^\text{dec}_c}

\newcommand\ep{\varepsilon}

\newcommand\re{\mathbb{R}}

\newcommand\cd{\mathcal{D}}
\newcommand\ce{\mathcal{E}}
\newcommand\cg{\mathcal{G}}

\newcommand\cj{\mathcal{J}}

\newcommand\cn{\mathcal{N}}

\newcommand\cx{\mathcal{X}}

\newcommand\cf{\mathcal{F}}

\newcommand{\co}{\mathcal{O}}
\newcommand{\B}{\mathbb{B}}
\newcommand{\bs}{\mathbb{S}}
\newcommand{\N}{\mathbb{N}}

\newcommand{\id}{\mathbb{I}}
\newcommand{\ex}{\mathbb{E}}

\DeclarePairedDelimiter{\ceil}{\lceil}{\rceil}
\newcommand{\norm}[1]{\left\lVert#1\right\rVert}

\newcommand{\iprod}[2]{\ensuremath{\left\langle #1,\,#2 \right\rangle}}
\newcommand*{\QEDB}{\hfill\ensuremath{\square}}%
\newcommand*{\pQEDB}{\hfill\ensuremath{\blacksquare}}

{\bf}{\it}

\def\BibTeX{{\rm B\kern-.05em{\sc i\kern-.025em b}\kern-.08em
    T\kern-.1667em\lower.7ex\hbox{E}\kern-.125emX}}
\begin{document}

\title{\textbf{The Deceptive Bandit Problem:\\ Exploratory Coupling and the Fragility of Multi-Agent Learning}}

% \title{\textbf{The Deceptive Bandit Problem: On the Fragility of Multi-Agent Learning via Exploratory Coupling}}

\author{Michael Tang, Mahmoud Abdelgalil, Jorge I. Poveda
\thanks{M. Tang and J. I. Poveda are with the Electrical and Computer Engineering Department at the University of California, San Diego, La Jolla, CA, USA. M. Abdelgalil is with the Mechanical and Aerospace Engineering Department at the University at Buffalo, State University of New York, Buffalo, NY, USA. }
\thanks{This work was supported in part by NSF grants DGE-2545911, CMMI 2228791, and the DARPA grant HR0011-25-3-0225.}
}

\maketitle

\begin{abstract}
Randomized exploration is central to bandit learning, multi-agent reinforcement learning, and zeroth-order policy search, yet its independence and privacy are usually only treated as technical assumptions. We show that these properties are critical for security purposes and demonstrate how an adversarial agent can exploit privileged information on another agent’s exploration. We analyze a deceiver-victim pair in the minimal two-player strongly monotone setting, where a deceptive player obtains leaked signals that are merely \emph{correlated} with the victim's exploration. We show that, by coupling their own exploratory action with this information, the deceptive player injects an externality that steers the learning dynamics to a new steady state, called the \emph{deceptive Nash equilibrium (DNE)}. We prove that the deceptive bandit learning (DBL) dynamics converge to an arbitrarily small neighborhood of the DNE while retaining optimal convergence rates. Interestingly, our analysis attains these optimal rates while relaxing \tcb{second-order smoothness} conditions from standard bandit optimization literature. We characterize conditions under which deception strictly shifts the steady state and its effect on the deceiver's cost, illustrating the results in a \tcb{resource-allocation game}.

% A long standing assumption in multi-agent bandit learning is that each agent explores independently of its peers. Herein, we challenge this premise by considering an asymmetric setting where an adversarial agent intercepts and exploits privileged information regarding another agent's exploration policy.
\end{abstract}

% \begin{IEEEkeywords}
% feedback optimization, singular perturbations, stochastic systems, covariance control.
% \end{IEEEkeywords}

\section{Introduction}
\label{sec:introduction}
% We have \cite{11048686}

%
Multi-agent learning has attracted significant attention in recent years with applications spanning many domains. In many of these applications, agents do not have access to gradient information and must instead rely solely on \emph{bandit feedback} \cite{flaxman2005online}, i.e., measurements of their cost/payoff functions, to update their policies. This has motivated a rich body of literature on \emph{multi-agent bandit learning}, where agents simultaneously explore their action spaces and exploit observed costs/payoffs to find optimal policies. Example applications include resource allocation \cite{landgren2021distributed}, distributed and networked systems \cite{liu2010distributed,shahrampour2017multi,zhu2023distributed}, robotics \cite{sankararaman2019social}, and financial decision-making \cite{misra2019dynamic,mueller2019low}. 
More broadly, randomized exploration is a basic ingredient of many learning methods, including stochastic approximation \cite{119632}, multi-agent reinforcement learning \cite{littman1994markov}, and zeroth-order policy search \cite{malik2019derivative}. 
If agents learn through bandit feedback, the convergence to the Nash equilibrium (NE) \cite{marden2009payoff,tatarenko2018learning,pang2020distributed} relies on the assumption that the exploratory policies used by different agents are \emph{independent} and \emph{private}. However, in an open, networked, or competitive environment, such an assumption is fragile; an adversarial agent can \emph{intercept} or \emph{infer} privileged information about the exploratory actions of other agents \cite{vamvoudakis2025deception,ye2021differentially, pmlr-v125-boursier20a}.
\emph{\textbf{Contributions:}} We expose a security vulnerability in exploration-driven multi-agent learning: information correlated with an agent’s exploratory behavior can be amplified into a persistent bias that changes the long-term outcome of learning. To isolate the core mechanism, we formalize it in two-player strongly monotone games, where a deceptive player modifies only its own exploration using a noisy signal correlated with the oblivious player’s current dither; extensions to $N$-player settings are also discussed. This unilateral coupling induces an externality in the oblivious player’s gradient estimate, scaled by the dither-signal correlation term. We prove that the resulting \emph{deceptive bandit learning} (DBL) dynamics converge to a small neighborhood of a new steady state, a \emph{deceptive} Nash equilibrium (DNE), with non-asymptotic mean-square guarantees. Our analysis attains the optimal convergence rates from \cite{9992950} while relaxing their \tcb{second-order smoothness} conditions on the players' costs; to our knowledge, this relaxation is also novel in the deception-free setting. We characterize when deception strictly shifts the equilibrium and provide geometric conditions under which this shift improves the deceiver's steady-state cost. Conversely, we also outline some defense strategies against deception that eliminate the induced bias. We then illustrate our main results on a resource-allocation game.

\emph{\textbf{Related Literature:}} Other works that study the vulnerabilities of multi-agent learning algorithms include utility measurement poisoning \cite{yao2025single, wu2023reward}, adversarial policies \cite{gleave2019adversarial}, and continuous-time equilibrium seeking \cite{11048686,tang2025stochastic}. Unlike \cite{yao2025single, wu2023reward, gleave2019adversarial}, the deceiver in our setting never tampers with the victim's cost measurements, but manipulates learning solely through its own feasible exploratory action. In \cite{11048686}, the deterministic sinusoidal probe structure means that knowledge of the victim’s exploration frequency suffices to reproduce its exploration signal; \cite{tang2025stochastic} requires exact real-time measurements of an OU-driven probe. Both study unconstrained continuous-time dynamics and establish local practical asymptotic guarantees under multiple vanishing parameter limits. In contrast, our setting does not require exact reconstruction or real-time measurements of the exploration signal; mere correlation suffices, as our deceiver receives only a noisy, possibly intermittent, randomly transformed, history-adaptive signal correlated with the current dither. We show that the conditional cross-covariance between leakage and exploration exactly governs the induced bias in the victim’s gradient estimator. For constrained, discrete-time bandit play, this yields, to our knowledge, the first non-asymptotic analysis of adversarial exploration coupling in multi-agent learning, with mean-square convergence guarantees.

\section{Preliminaries}
\label{sec_prelim}
Given a matrix $A\in\re^{n\times m}$, we use $\norm{A}$ to denote the operator norm of $A$. If $A$ is a vector, then $\norm{A}$ evaluates to the Euclidean norm of $A$. If $f:\re^n\to\re^m$ is differentiable, then $\nabla f:\re^n\to\re^{n\times m}$ denotes the transpose of the Jacobian of $f$. We use $\bs^n=\{s\in\re^n : \norm{s}=1\}$ and $\B^n=\{s\in\re^n : \norm{s}\le 1\}$ to denote the unit sphere and ball in $\re^n$, respectively. We denote the ball around $p\in\re^n$ of radius $r>0$ by $\B_r(p)=\{s\in\re^n : \norm{p-s}\le r\}$. For a closed convex set $E\subset\re^n$, the projection operator is defined as $\pr_E(x)=\argmin_{p\in E} \norm{x-p}$. We use $\text{int} ~ E$ to denote the interior of the set $E$.

\section{Problem Setup}
\label{sec_prob}
{We focus on a two-player game, the minimal setting containing a deceiver-victim pair; the generalization to $N$-player settings is discussed in Remark \ref{hetero}.} Player $i$ selects an action $x_i$ from an action space $\cx_i\subset\re^n$ with the goal of minimizing their cost $J_i:\cx_1\times\cx_2\to\re$. The joint action profile is given by $x=(x_1, x_2)\in\cx$, where $\cx=\cx_1\times\cx_2$ denotes the game's joint action space. A point $x^*$ is said to be a \emph{Nash equilibrium (NE)} if each player has no incentive to unilaterally change their action, i.e.,
\begin{equation}
    x_i^*\in\argmin_{x_i\in\cx_i}J_i(x_i;x_{3-i}^*),\quad i\in\{1,2\},
\end{equation}
where we use $J_i(x_i;x_{3-i})$ and $J_i(x)$ interchangeably. Given the costs, the \emph{pseudogradient} of the game is defined as:
\begin{equation}
    \cg(x):=\begin{bmatrix}
        \nabla_1 J_1(x) \\ \nabla_2 J_2(x)
    \end{bmatrix},        
\end{equation}
where $\nabla_i(\cdot)$ represents differentiation with respect to $x_i$.
We impose the following assumptions on the game parameters:
%
%\vspace{-5pt}
\begin{assumption}\label{assumps}
    For each $i\in\{1,2\}$ and $j\in\{1,2\}$, the following hold:
    \begin{enumerate}
    \item The set $\cx_i$ is compact and convex, and there exists $p_i\in\cx_i$, $r_i>0$ such that $\B_{r_1}(p_1)\times \B_{r_2}(p_2)\subset \cx$.
    \item The function $J_i(x)$ is continuously differentiable on an open neighborhood of $\cx$, and there exists $L_{ij}>0$ such that the function $\nabla_j J_i(x)$ is $L_{ij}$-Lipschitz on $\cx$.
    \item There exists $\gamma>0$ such that the operator $\cg(\cdot)$ is $\gamma$-strongly monotone over $\cx$, i.e., the following holds
    \begin{equation}\label{strongm}
        \iprod{\cg(x)-\cg(y)}{x-y}\ge \gamma\norm{x-y}^2,
    \end{equation}
    for all $x,y\in\cx$.\QEDB
    \end{enumerate}
\end{assumption}
%\vspace{-5pt}
Conditions 1-3 are standard in continuous-action bandit learning \cite{flaxman2005online,yan2023fast,bravo2018bandit}; note that only first-order smoothness is imposed. Games with pseudogradients that satisfy \eqref{strongm} are referred to as \emph{$\gamma$-strongly monotone}, which are ubiquitous in the game theory literature \cite{yan2023fast, ba2025doubly, jordan2025adaptive}.
It is well-known that in games where Assumption \ref{assumps} holds, the Nash equilibrium $x^*$ is unique and satisfies the following variational inequality
\begin{equation}\label{varineq}
    \iprod{\cg(x^*)}{x-x^*}\ge 0,\quad\forall x\in\cx.
\end{equation}
Hence, we say that an action profile $x^*\in\cx$ is a ``\emph{Nash equilibrium of $(\cg, \cx)$}" if \eqref{varineq} holds.

\subsection{Standard Bandit Feedback Algorithm}
For our work, we will study the bandit feedback-based algorithm from \cite{flaxman2005online,bravo2018bandit, 9992950}, which can be implemented by the agents to learn the Nash equilibrium by only using measurements of their cost functions. At each time step $k$, player $i$ samples a vector $\mu_{i,k}\in\bs^n$ uniformly at random and computes the following intermediate exploration step
\begin{equation}\label{explore}\bar{x}_{i,k}=x_{i,k}+\delta\mu_{i,k},
\end{equation}
where $\delta\in (0,\min_i r_i)$ is the search radius and $x_{i,k}$ is the action of player $i$ at the $k$-th iteration. For notational convenience, we assume all players use the same perturbation radius $\delta$. This is quite standard \cite{bravo2018bandit, 9992950}, and the case of heterogeneous $\delta_i$ is discussed in Remark \ref{hetero}. Player $i$ then performs the action $\bar{x}_{i,k}$, observes their resulting cost $J_i(\bar{x}_k)$ (with $\bar{x}_k=(\bar{x}_{1,k}, \bar{x}_{2,k})$), and updates their action:
\begin{equation}\label{exploit}
    x_{i,k+1}=\pr_{\cx_i^\delta}\left(x_{i,k}-\eta_k\frac{n}{\delta}J_i(\bar{x}_k)\mu_{i,k}\right),
\end{equation}
where $\cx_i^\delta=p_i+(1-\frac{\delta}{r_i})(\cx_i-p_i)$; intuitively, this is a homothetic contraction of $\cx_i$ towards the point $p_i$ (from Assumption \ref{assumps}). The projection operator in \eqref{exploit} ensures that the action profile at the next iterate, specified in \eqref{explore}, remains feasible. This type of feasibility adjustment, which is common in bandit learning \cite{bravo2018bandit}, generalizes the contraction $(1-\delta)\cx_i$ from \cite{9992950} and allows us to consider action spaces that need not contain the origin. Both players use a common learning rate $\eta_k$ that satisfies the standard assumption \cite{119632}:
%
%\vspace{-5pt}
\begin{assumption}\label{as_eta}
    The sequence $\{\eta_k\}_{k=1}^\infty$ satisfies:
    \begin{equation}\label{etak}       \lim_{k\to\infty}\eta_k=0,\quad\sum_{k=1}^\infty \eta_k=\infty,
    \end{equation}
    and $\eta_k>0$ for each $k\in\N$.\QEDB
\end{assumption}
%\vspace{-5pt}
%
The algorithm \eqref{explore}-\eqref{exploit} has been covered extensively in the \emph{bandit optimization} literature \cite{flaxman2005online,9992950, bravo2018bandit, ba2025doubly}, and the main intuition behind it is in the signal $\frac{n}{\delta}J_i(\bar{x}_k)\mu_{i,k}$ from \eqref{exploit}, which serves as a single-point stochastic gradient estimator. More precisely, its conditional expectation equals the gradient of a suitably smoothed version of $J_i$, which differs from $\nabla_i J_i(x_k)$ by an $\co(\delta)$ term under Assumption \ref{assumps} (where $x_k=(x_{1,k}, x_{2,k})$). Hence, \eqref{exploit} can be interpreted as a noisy gradient descent update on a slightly perturbed game. Under Assumptions \ref{assumps}-\ref{as_eta} and independent exploration, these dynamics are known to efficiently learn the unique NE.
% Classical results in bandit learning establish that, if Assumptions \ref{assumps}-\ref{as_eta} hold, then the iterates of \eqref{explore}-\eqref{exploit} satisfy $\limsup_{k\to\infty}\ex\norm{x_k-x^*}^2=\co(\delta^2)$ \cite{9992950}.

\section{Bandit Feedback with Deception}
\label{sec_bndtdeception}
The convergence guarantees of exploration-based algorithms, such as \eqref{explore}-\eqref{exploit}, rely on the premise that agents conduct their exploration independently and without interference. However, in competitive, adversarial, or asymmetric environments, these independence and privacy assumptions can break down \cite{vamvoudakis2025deception,ye2021differentially, pmlr-v125-boursier20a}.  In games, the signals available to the players form an information structure separate from the actions and costs \cite{bergemann2016bayes}. Models of espionage or information leakage study settings where one player receives noisy or partial information about an opponent's selected strategy or realized randomization \cite{solan2004games, alon2013adversarial, alvim2017information}. These works motivate the need to examine analogous asymmetries in bandit learning, where the leaked randomization is now the mechanism by which each player learns.

Specifically, we consider a setting where player 2 intercepts or infers privileged information regarding player 1's search policy. If player 2 integrates this information into only their own actions, they can manipulate the gradient estimate that player 1 uses to learn. Here, we call player 2 the \emph{deceptive} player, while player 1 is the \emph{oblivious} player. This section formalizes the structure of the leaked information, the DBL dynamics, and their long-term effects on learning.

% We consider a setting where player 2 possesses the ability to intercept, anticipate, or observe privileged information regarding player 1's instantaneous search policy. In this case, player 2 will be referred to as the \emph{deceptive} player, while player 1 is the \emph{oblivious} player. As we show later, if player 2 integrates this information into their own actions, they can manipulate the cost landscape that player 1 is trying to learn. 

% In this section, we formalize this adversarial capability. We define the structure of the leaked information from the oblivious player and introduce the \emph{deceptive dynamics} that result from player 2's exploitation of this information. We then demonstrate how player 2's manipulation alters the underlying emergent behavior of the game, shifting the system away from the nominal Nash equilibrium and towards a new steady state: the deceptive Nash equilibrium.

\subsection{Information Asymmetry}

Suppose that at each iteration $k$, each player selects a search direction $\mu_{i,k}\in\bs^n$ independently and uniformly at random, but player 2 also receives a signal $\Theta_{1,k}\in\re^n$ possibly containing privileged information regarding $\mu_{1,k}$. We use $\cf_k$ to denote the pre-round history up to stage $k$, and assume that $\Theta_{1,k}$ satisfies the following structural condition:
\vspace{-15pt}
\begin{assumption}\label{as_th}
    \tcb{For $k\in\N\cup\{0\}$, the signal $\Theta_{1,k}\in\re^n$ takes the form
    \begin{equation}\label{th1k}        \Theta_{1,k}=\Lambda_k\mu_{1,k}+w_k,
    \end{equation}
    where, conditional on $\cf_k$, the following hold almost surely:
    \begin{enumerate}
        \item The directions $\mu_{1,k}$ and $\mu_{2,k}$ are independent and uniformly distributed on $\bs^n$. Moreover, the pair $(\Lambda_k, w_k)$ is conditionally independent of $(\mu_{1,k}, \mu_{2,k})$.
        \item There exists a matrix $\Gamma\in\re^{n\times n}$, independent of $k$, such that $\ex_k[\Lambda_k]=\Gamma$ a.s, where $\ex_k[~\cdot~]:=\ex[~\cdot~ | \cf_k]$.
        \item There exists constants $\bar{\Lambda}, \bar{w}>0$ such that $\norm{\Lambda_k}\le\bar{\Lambda}$  and $\norm{w_k}\le\bar{w}$ a.s.
    \end{enumerate}
    All random variables are defined on the filtered probability space $(\Omega, \cf, \{\cf_k\}_{k\ge 0},\mathbb{P})$, with pre-round history given by the $\sigma$-algebra $\cf_k=\sigma(x_0, \{\mu_{1,\ell}, \mu_{2,\ell},\Lambda_\ell, w_\ell\}_{\ell < k})$.\QEDB
    }
\end{assumption}
%\vspace{-5pt}

Assumption \ref{as_th} captures a range of interception scenarios, such as channel leaks, noisy sensing, and intermittent measurements. In a practical context, $\Theta_{1,k}$ represents player 2's imperfect measurement of player 1's exploratory behavior. 
% In contrast with \cite{11048686, tang2025stochastic}, equation \eqref{th1k} does not assume perfect \tcb{real-time} knowledge of the oblivious player's dithers.

%Note: this formulation does not require that the deceptive player always receives the signal, since the case of probabilistic deception can be incorporated by multiplying a Bernoulli RV to Lambda_k and w_k

\tcb{The structure of \eqref{th1k} isolates two fundamental sources of imperfection in this observation. The random matrix $\Lambda_k$ models the transformation, attenuation, or scaling of the intercepted signal. The existence of $\Gamma$ ensures that the dither-signal cross-covariance is well-behaved. Meanwhile, $w_k$ captures ambient measurement noise or channel interference, which we do not require to be zero mean. More generally, since $\mu_{1,k}=\frac{1}{\delta}(\bar{x}_{1,k}-x_{1,k})$, the signal $\Theta_{1,k}$ may also be a noisy measurement of the played action $\bar{x}_{1,k}$ rather than the actual dither. Notably, the formulation \eqref{th1k} also allows for intermittently available measurements. Indeed, suppose that at stage $k$ the signal is only intercepted with probability $\rho\in (0,1]$. Taking $\Lambda_k=b_k \widetilde{\Lambda}_k$ and $w_k=b_k\widetilde{w}_k$ with $b_k | \cf_k\sim\text{Bernoulli}(\rho)$ (independent of the dithers and $\widetilde{w}_k, \widetilde{\Lambda}_k$) gives $\Gamma=\rho\widetilde{\Gamma}$, where $\widetilde{\Gamma}=\ex_k[\widetilde{\Lambda}_k]$. Hence, missed or late interceptions merely attenuate the mean dither-signal cross-covariance by $\rho$.}
% This implies that player 2 can successfully deceive even with poorly calibrated and persistently biased observations, provided the noise is bounded with probability one.

\subsection{Bandit Feedback Algorithm with Deception}

With the information asymmetry formally defined, we now construct the modified learning dynamics. We assume player 1 remains oblivious to deception and continues to execute the standard bandit dynamics given below
\begin{subequations}\label{dx1}
    \begin{align}        \bar{x}_{1,k}&=x_{1,k}+\delta\mu_{1,k}\\
    x_{1,k+1}&=\pr_{\cx_1^\delta}\left(x_{1,k}-\eta_k\frac{n}{\delta}J_1(\bar{x}_k)\mu_{1,k}\right).
    \end{align}
\end{subequations}
Meanwhile, suppose that player 2 exploits the leaked signal by actively coupling their exploration with player 1's search direction. Before the stage $k$ costs are realized, player 2 injects the signal $\Theta_{1,k}$ into their exploration step with a tunable gain $c$. The dynamics implemented by player 2 are:
\begin{subequations}\label{dx2}
    \begin{align}        \bar{x}_{2,k}&=x_{2,k}+\delta\mu_{2,k}+c\delta\Theta_{1,k}\label{dx2h}\\
    x_{2,k+1}&=\pr_{\cx_2^{{\delta_c}}}\left(x_{2,k}-\eta_k\frac{n}{\delta}J_2(\bar{x}_k)\mu_{2,k}\right),
    \end{align}
\end{subequations}
where ${\delta_c}=\delta(1+|c|(\bar{\Lambda}+\bar{w}))$, and we assume that $\delta,c$ are chosen such that $\delta_c<r_2$. The formulation of $\delta_c$ guarantees that even in the worst-case realization of $\Lambda_k$ and $w_k$, player 2's exploratory action remains feasible. We will call system \eqref{dx1}-\eqref{dx2} the \emph{deceptive bandit learning (DBL) dynamics}.

\begin{figure}[t!]
  \centering
    \vspace{0.2cm}\includegraphics[width=0.46\textwidth]{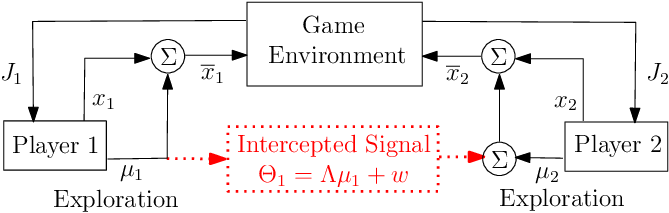}
    \caption{\small Exploration leakage and coupling in a two-player game with bandit feedback.}\label{sblock}
    \vspace{-0.4cm}
\end{figure}

The subtlety behind the proposed deception mechanism can be observed through the expected gradient signal, given by \tcb{$\widehat{g}_{i,k}=\frac{n}{\delta}J_i(\bar{x}_k)\mu_{i,k}$}. Computing the smoothed average of $\widehat{g}_{i,k}$ around \tcb{$x_k$ (details in the Appendix) yields:
\begin{subequations}\label{expg}
\begin{align}
        \ex_k[\widehat{g}_{1,k}]&=\nabla_1 J_1(x_k)+c\Gamma^\top \nabla_2 J_1(x_k)+\beta_{1,k}^\delta\label{expg1}\\
        \ex_k[\widehat{g}_{2,k}]&=\nabla_2 J_2(x_k)+\beta_{2,k}^\delta,
    \end{align}
\end{subequations}
where $\beta_{1,k}^\delta$ and $\beta_{2,k}^\delta$ are $\cf_k$-measurable $\mathcal{O}(\delta)$ smoothing biases, with bounds provided in Lemma \ref{lem_o} of the Appendix; such biases are present even in the nominal ($c=0$) algorithm. The distinctive effect of deception is the induced externality $c\Gamma^\top \nabla_2 J_1(x_k)$, which does not vanish with $\delta$.} The system \eqref{expg} reveals that while player 2 still learns their gradient up to the $\co(\delta)$ bias, player 1 now learns a flawed estimate affected by this externality.

\vspace{0.1cm}
\begin{remark}\label{rem1}
\tcb{The effect of deception is governed by the conditional cross-covariance between the oblivious player’s dither and the intercepted signal:
    \begin{equation*}       C_{\Theta,k}:=n\ex_k[\mu_{1,k}\Theta_{1,k}^\top].
    \end{equation*}
This correlation introduces the externality $cC_{\Theta,k}\nabla_2 J_1(x_k)$ into the conditional mean of the oblivious player's gradient estimator. Under Assumption \ref{as_th}, $C_{\Theta,k}=\Gamma^\top$ almost surely. Hence, although the conditional distribution of the signal may adapt to the pre-round history, its fixed cross-covariance produces the same deceptive bias at every iteration. Conversely, if $C_{\Theta,k}=0$, the intercepted signal cannot generate this equilibrium-shifting bias. 
This observation also reveals possible defense strategies that can be implemented by the victim, which include desynchronizing the game, applying a randomized zero-mean mask to the dither, etc. An in-depth study of defense strategies is left for future work.\QEDB
}
\end{remark}

\vspace{0.1cm}
\begin{remark}\label{hetero}
    If players use different exploration radii $\delta_1, \delta_2$ with $\delta_i=a_i\delta$, where $a_i>0$, and \eqref{dx2h} uses $c\delta_2\Theta_{1,k}$, then the externality in \eqref{expg1} becomes $c(a_2/a_1)\Gamma^\top \nabla_2 J_1(x_k)$, i.e., the effective gain is ${c}_\text{eff}=c(a_2/a_1)$. The computation in \eqref{expg} also covers the $N$-player case where player $d$ deceives a player $v$; dithers are independent, so the injected externality becomes $c\Gamma^\top\nabla_d J_v(x_k)$ and only affects player $v$'s gradient estimator (i.e., the $v$ block of the pseudogradient). \QEDB
\end{remark}

Motivated by \eqref{expg}, we propose the notion of a \emph{deceptive pseudogradient}, which is given by
\begin{equation}\label{decsg}
    \gdec(x)=\begin{bmatrix}
        \nabla_1 J_1(x)+c\Gamma^\top\nabla_2 J_1(x) \\ \nabla_2 J_2(x)
    \end{bmatrix}.
\end{equation}
For the modified learning dynamics to converge, $\gdec$ must remain well-behaved. Specifically, the injection of the deceptive externality $c\Gamma^\top \nabla_2 J_1(x)$ must not destroy the monotonicity of the nominal game. By leveraging the Lipschitz continuity of $\nabla_2 J_1(\cdot)$ obtained from Assumption \ref{assumps}, it is straightforward to verify that $\gdec$ remains strongly monotone as long as the following condition is satisfied:
\begin{equation}\label{cbd}
    \gamma_c:=\gamma-|c| L_{12} \norm{\Gamma}>0,
\end{equation}
where $L_{12}$ is the Lipschitz constant of $\nabla_2 J_1(\cdot)$. If \eqref{cbd} holds, there exists a unique Nash equilibrium of $(\gdec, \cx)$, which we call the \emph{deceptive Nash equilibrium (DNE)}. This is formalized in the following definition:
%\vspace{-5pt}
\begin{defn}
   Suppose that $\gdec$, given in \eqref{decsg}, is strongly monotone. The deceptive Nash equilibrium, denoted $x_c^*$, is the unique point in $\cx$ that satisfies the following
    \begin{equation}\label{videc}
    \iprod{\gdec(x_c^*)}{x-x_c^*}\ge 0,
\end{equation}
for all $x\in\cx$.\QEDB
\end{defn}
\vspace{-5pt}

Equations \eqref{decsg}, \eqref{videc} identify a new potential steady state induced by the deception mechanism. This raises two natural questions: do the DBL dynamics approach this DNE, and how does the DNE compare with the nominal NE? We answer these questions next.
\section{Main Results}
\label{sec_main}
We establish convergence guarantees for the DBL dynamics, then characterize when deception shifts the steady state and benefits the deceiver.
\subsection{Convergence Analysis}
Because the coupled gradient estimators now track $\gdec$, the first question is how closely the DBL iterates approach the DNE and at what rate. We provide both guarantees in the following theorem:
\begin{thm}\label{thm_conv}
    Consider the DBL dynamics \eqref{dx1}-\eqref{dx2} under Assumptions \ref{assumps}-\ref{as_th}, where \eqref{cbd} holds.
    Then, there exists $R>0$ such that, for any \tcb{$x_0\in\cx_1^\delta\times\cx_2^{\delta_c}$} and $\delta$ satisfying
    \begin{equation}\label{delta_bd}
        0<\delta<\min\left\{r_1,\frac{r_2}{1+|c|(\bar{\Lambda}+\bar{w})}\right\}:=r_c,
    \end{equation}
    the following bound holds
    \begin{equation}\label{limxk}
        \limsup_{k\to\infty}\ex[\norm{x _k-x_{c}^*}^2]\le R\delta^2.
    \end{equation}
    If the step sizes also satisfy $\eta_0\ge 0$ and $\eta_k=\frac{q}{\gamma_c k}$ \tcb{for $k\ge 1$}, where $q>1$ and $\gamma_c=\gamma-L_{12}|c|\norm{\Gamma}$, then there also exists $M>0$ such that the following holds
    \begin{equation}\label{limxkk}
        \ex[\norm{x _k-x_{c}^*}^2]\le \frac{M}{\delta^2 k}+R\delta^2,
    \end{equation}
    for all $k\in\N$ and $\delta\in (0,r_c)$.
\end{thm}
\vspace{0.1cm}
\vspace{-5pt}
Thus, while player 1 remains oblivious and continues to execute the nominal bandit update, the joint dynamics converge to a neighborhood of the shifted DNE. From \eqref{limxkk}, we see that the DBL dynamics inherit the same $\co(d^2/{\ep^2})$ complexity guarantees from \cite{9992950}, where $d=2n=\text{dim}~\cx$:
%\vspace{-5pt}
\begin{cor}\label{cor_complex}
    Under the conditions of Theorem \ref{thm_conv} for which \eqref{limxkk} holds, for any $\ep>0$ sufficiently small, choosing $\delta^2=\frac{\ep}{(2R)}$ and $ k\ge\frac{4MR}{\ep^2}$ yields $\ex[\norm{x_k-x_c^*}^2]\le \ep$, where $M, R$ are given in Section \ref{convproof}. Consequently, the DBL dynamics attain the same $\co(d^2/\ep^2)$ complexity as the deception-free bandit algorithm from \cite{9992950}.
\end{cor}
%\vspace{-3pt}
\begin{remark} Corollary \ref{cor_complex} demonstrates that the dynamics remain efficient even under adversarial manipulation.
By treating the smoothing bias in $\ex_k[\widehat{G}_k]$ as a bounded additive perturbation (see the proof of Theorem \ref{thm_conv} in Section \ref{convproof}), we establish these optimal rates while relaxing the assumption of a Lipschitz Jacobian for the pseudogradient, which \cite{9992950} needed for proving $\co(d^2/{\ep^2})$ complexity. The same bounded-perturbation argument extends to $N$-player games. To our knowledge, our derived relaxation is also novel in the deception-free ($c=0$) setting.\QEDB 
\end{remark}
\vspace{3pt}
\begin{remark}
    Deception is not free. If $c$ satisfies \eqref{cbd}, then the strong monotonicity of $\gdec$, with \eqref{varineq} and \eqref{videc}, implies:
    \begin{equation}\label{xcxbd}
        \norm{x_c^*-x^*}\le \frac{|c|}{\gamma_c}\norm{\Gamma^\top\nabla_2 J_1(x^*)}.
    \end{equation}
 So while larger $|c|$ permits larger equilibrium displacement and potential long-term expected cost improvement, it also shrinks $\gamma_c$, contracts player 2's feasible set $\cx_2^{\delta_c}$, and inflates the smoothing bias bounds in \eqref{bias_bd}; this effectively degrades the constants $M$ and $R$ in Theorem \ref{thm_conv}. Hence, selecting $c$ itself presents a design trade-off for the deceiver. \QEDB
\end{remark}

\subsection{Characterization of the DNE and Immunity Conditions}
We established convergence to a neighborhood of the DNE, which is characterized by the variational inequality \eqref{videc}. It is evident from \eqref{videc} that the steady state conditions are altered under deception, but a fundamental question remains: when is the DNE actually different from the nominal NE? Equation \eqref{xcxbd} bounds the magnitude of this shift, but we now ask when it is nonzero. Furthermore, in what situations does deception actually improve the cost of the deceiver?

To isolate the effect of the deception from the complexities of active constraints, our vulnerability characterization focuses on interior nominal NE, which is a practical operating condition in many real-world settings \cite{6060862, 9992906, perolat2021poincare}. Specifically, we find that interior NE are locally steerable via deception whenever the externality $\Gamma^\top\nabla_2 J_1(\cdot)$ is nonzero at the nominal NE. Conversely, the equilibrium is invariant whenever $\Gamma^\top \nabla_2 J_1(x^*)=0$, even if $x^*$ is not necessarily an interior NE. Thus, the zero-correlation condition $\Gamma=0$ from Remark \ref{rem1} provides a sufficient universal defense against this particular deception mechanism. Meanwhile, nonzero
$\Gamma$ can still yield immunity if
$\operatorname{range}(\Gamma)\perp\nabla_2J_1(x^\ast)$, i.e., when the leakage carries no information about the externality.
We formalize these ideas in the following theorem:

\begin{figure*}[t!]
    \centering

    \begin{minipage}[c]{0.337\textwidth}
        \centering
        \includegraphics[
            width=\linewidth,
            trim=2mm 1mm 2mm 1mm,
            clip
        ]{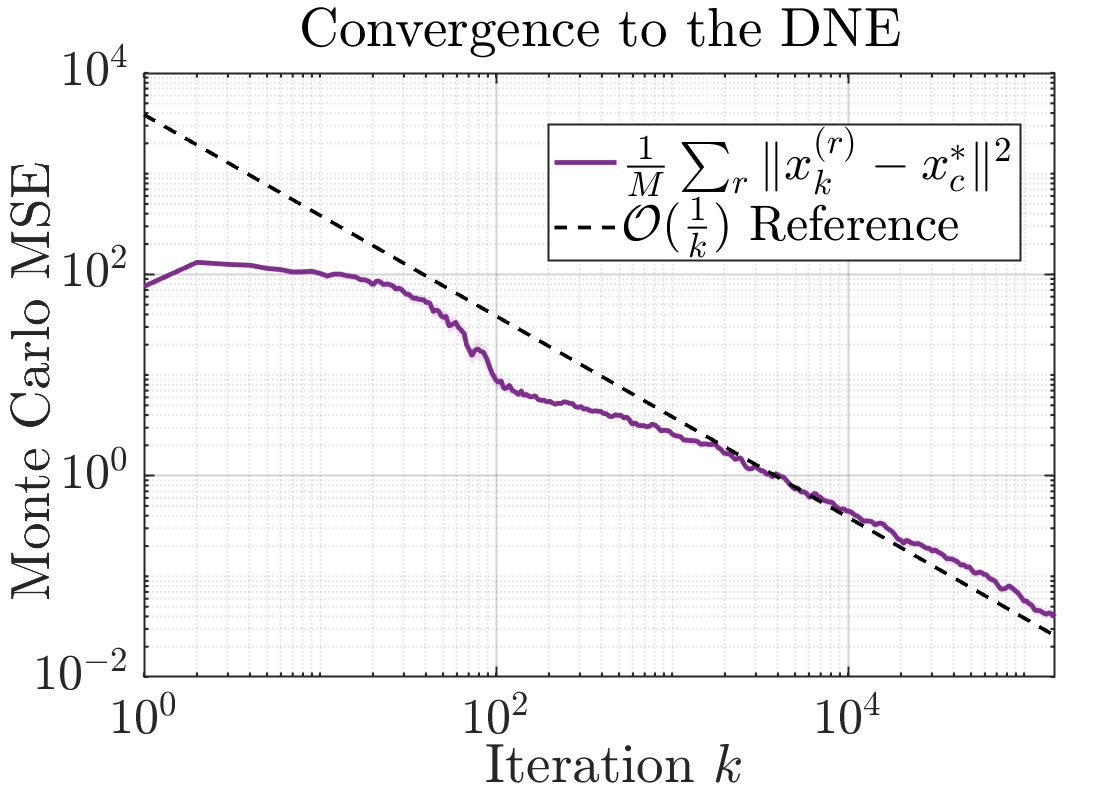}
    \end{minipage}\hspace{-0.35cm}%
    \begin{minipage}[c]{0.337\textwidth}
        \centering
        \includegraphics[
            width=\linewidth,
            trim=2mm 1mm 2mm 1mm,
            clip
        ]{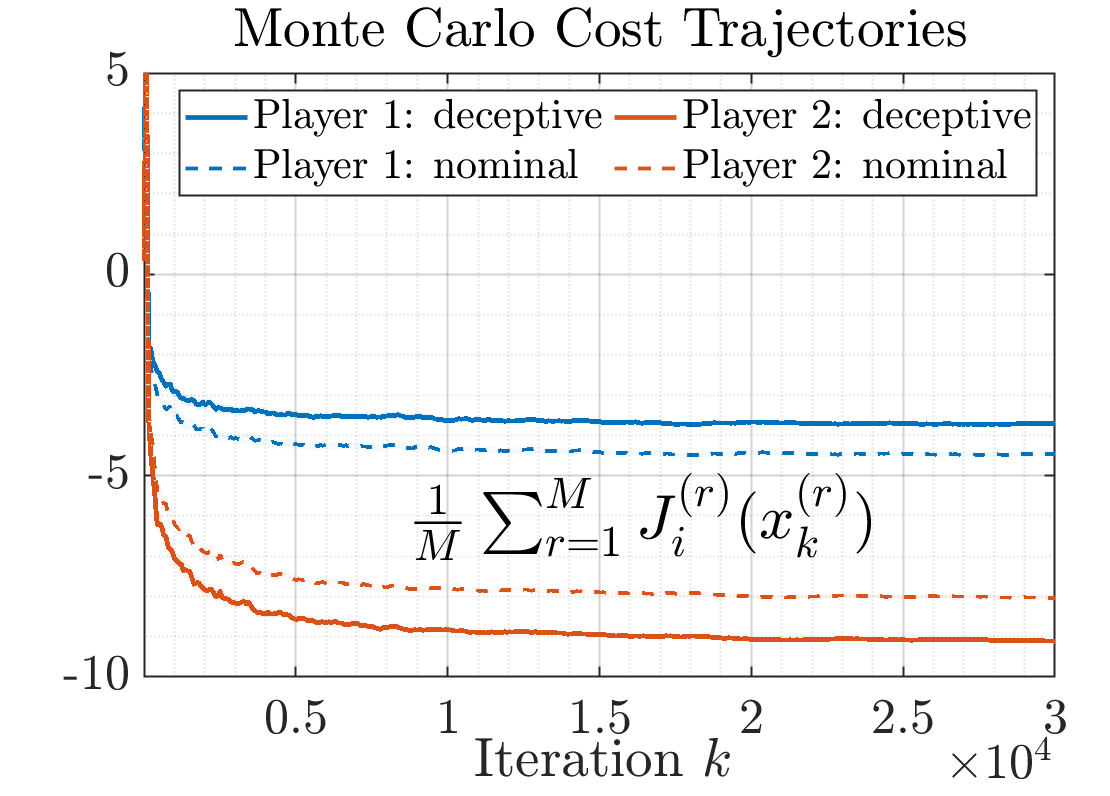}
    \end{minipage}\hspace{-0.15cm}%
    \begin{minipage}[c]{0.337\textwidth}
        \centering
        \includegraphics[
            width=\linewidth,
            trim=2mm 1mm 2mm 1mm,
            clip
        ]{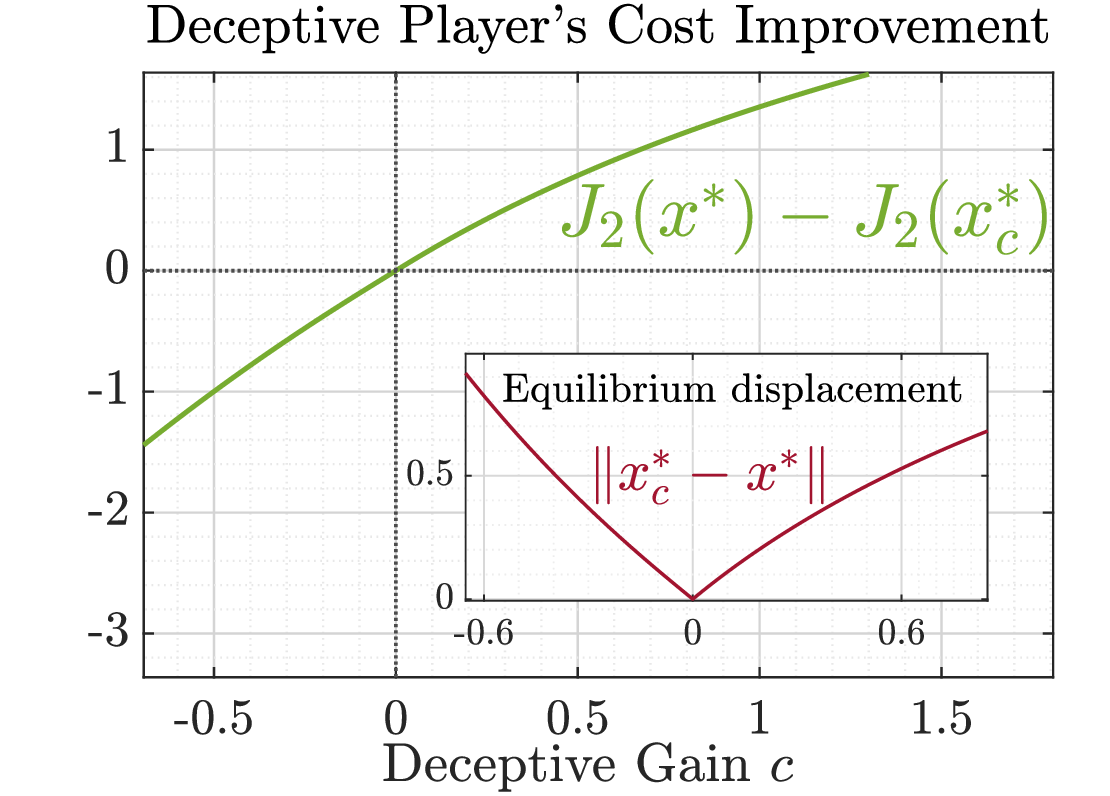}
    \end{minipage}

    \vspace{-0.1cm}
    \caption{\small
    Monte Carlo study over $M=300$ realizations.
    Left: convergence to a neighborhood of the DNE.
    Middle: nominal and deceptive cost trajectories.
    Right: the deceptive player's steady-state cost change as a
    function of \(c\), and the equilibrium
    displacement.}
    \label{figs}
    \vspace{-0.5cm}
\end{figure*}
%\vspace{-5pt}
\begin{thm}\label{thm_inj}
    Suppose that Assumptions \ref{assumps} and \ref{as_th} hold. If $x^*\in\text{int}~\cx$ and $\Gamma^\top\nabla_2 J_1(x^*)\neq 0$, then the mapping $c\mapsto x^*_c$ is injective on a neighborhood of 0. If instead $\Gamma^\top\nabla_2 J_1(x^*)= 0$, then $x_c^*=x^*$ for all $c$ where $\gdec$ is strongly monotone.
\end{thm}
\begin{remark}
The two immunity conditions in Theorem \ref{thm_inj} differ in robustness:
$\Gamma=0$ depends only on the leakage channel and is enforceable by
the oblivious player via Remark \ref{rem1}, while
$\operatorname{range}(\Gamma)\perp\nabla_2J_1(x^*)$ depends on the
nominal NE and is, hence, unreliable as a defense.\QEDB
\end{remark}
\subsection{Impact of Deception}
While Theorem \ref{thm_inj} tells us when the deceiver can steer the dynamics to a new steady state, it is common in practical scenarios that the deceiver is interested in improving their own well-being. If the mapping $c\to x_c^*$ is differentiable around 0, we can evaluate the local sensitivity of $J_2$ to $c$. Thus, the question of whether deception is profitable naturally reduces to characterizing this derivative and identifying when it does not vanish. These conditions are formalized below:
%\vspace{-3pt}
\begin{thm}\label{thm_imp}
    Suppose that Assumptions \ref{assumps} and \ref{as_th} hold, where $\gdec$ is continuously differentiable and $x^*\in\text{int}~\cx$. If the following holds:
    \begin{equation}\label{implicit}
     \begin{bmatrix} \nabla_1 J_2(x^*) \\ 0 \end{bmatrix}^\top \nabla\mathcal{G}(x^*)^{-\top} \begin{bmatrix} \Gamma^\top \nabla_2 J_1(x^*) \\ 0 \end{bmatrix}\neq 0,
    \end{equation}
    there exists a nonempty open interval $\ce\subset \re$ such that, for all $c\in \ce$, $\gdec$ is strongly monotone and $J_2(x_c^*)<J_2(x^*)$.
\end{thm}
%\vspace{-3pt}

Theorem \ref{thm_imp} provides some important geometric intuition behind the mechanics of an adversarially favorable shift. Naturally, for \eqref{implicit} to hold, the term $\Gamma^\top\nabla_2 J_1(x^*)$ cannot be zero. This parallels Theorem \ref{thm_inj}, which established that this nonzero externality was needed for injectivity of the mapping $c\mapsto x_c^*$. Without this condition, the DNE remains at $x^*$, which implies that the steady state cost function values will also be unchanged. Surprisingly, condition \eqref{implicit} also requires that the other cross derivative $\nabla_1 J_2(\cdot)$ is nonzero at the NE. Intuitively, this makes sense: if the deceiver's cost is completely insensitive to the victim's actions at the nominal equilibrium, shifting the victim's steady-state behavior yields no immediate benefit. Theorem \ref{thm_conv} then tells us that the cost improvements in Theorem \ref{thm_imp} can be realized in the long-term expected DBL trajectories whenever $\delta$ is sufficiently small.
\section{Proofs of Main Results}
\subsection{Proof of Theorem \ref{thm_conv}}\label{convproof}
Let $\cx^\delta_c:=\cx_1^\delta\times\cx_2^{\delta_c}$, and let $x_{c,\delta}^*$ denote the unique Nash equilibrium of $(\gdec, \cx^\delta_c)$. Define $\widehat{G}_k=\frac{n}{\delta}(J_1(\bar{x}_k)\mu_{1,k}, J_2(\bar{x}_k)\mu_{2,k})$.
   Letting $e_k=x_k-x^*_{c,\delta}$, for each $k\ge 0$ we have:
    \begin{align}
        \ex[\norm{x _k-x_{c}^*}^2]&\le 2 \ex\left[\norm{e_k}^2\right]+2\norm{x_{c,\delta}^*-x_c^*}^2.\label{limex}
    \end{align}
    \tcb{By construction, $x_k$ and $e_k$ are $\cf_k$-measurable for every $k$.} To upper bound \eqref{limex}, let us consider the first term, since the second term can be upper bounded with Lemma \ref{xcdbd} in the Appendix. Since $x_{c,\delta}^*\in\cx^\delta_c$, we can leverage the non-expansiveness of the projection operator to obtain:
    \begin{align*}
         &\ex_k\left[\norm{e_{k+1}}^2\right]\le \ex_k\left[\norm{x_k-\eta_k\widehat{G}_k-x_{c,\delta}^*}^2\right]\\
         &~~~~\le \ex_k\left[\norm{e_k}^2\right]+\eta_k^2\ex_k\left[\norm{\widehat{G}_k}^2\right]-2\eta_k\iprod{\ex_k\left[\widehat{G}_k\right]}{e_k}
         .
    \end{align*}
Let \tcb{$\Delta_k^\delta=(\beta_{1,k}^\delta, \beta_{2,k}^\delta)$, where $\beta_{1,k}^\delta$ and $\beta_{2,k}^\delta$} are from \eqref{expg} and satisfy the bounds given in Lemma \ref{lem_o} from the Appendix. Let \tcb{$\cj=\sup_{i\in\{1,2\}}\sup_{x\in\cx}|J_i(x)|$, which is finite}. Since $\ex_k[\widehat{G}_k]=\gdec(x_k)+\Delta_k^\delta$, we obtain:
\begin{align*}
    \ex_k\left[\norm{e_{k+1}}^2\right]
         &\le \ex_k\left[\norm{e_k}^2\right]-2\eta_k\ex_k\left[\iprod{\gdec(x_k)}{e_k}\right]\\
         &~~~ +\frac{2n^2\eta_k^2\cj^2}{\delta^2}-2\eta_k\ex_k\left[\iprod{\Delta_k^\delta}{e_k}\right] \text{ a.s.}
\end{align*}
By the definition of $e_k$ and $x_{c,\delta}^*$, we know that the following holds for all $x_k\in\cx_c^\delta$:
\begin{equation}\label{gcmon}
    \iprod{\gdec(x_k)}{e_k}\ge \gamma_c \norm{e_k}^2.
\end{equation}
We can also apply Young's inequality to obtain
\begin{align}\label{dyg}
    |\iprod{\Delta_k^\delta}{e_k}|&\le \frac{\gamma_c}{2}\norm{e_k}^2+\frac{1}{2\gamma_c}\norm{\Delta_k^\delta}^2.
\end{align}
By Lemma \ref{lem_o} in the Appendix, there exists $\alpha>0$ such that $\norm{\Delta_k^\delta}^2\le \alpha\delta^2$ almost surely for each $k$.
We can then combine \eqref{gcmon} with \eqref{dyg} and take the total expectation over $\cf_k$ to obtain
\begin{align}\label{exk1bd}
    \ex\left[\norm{e_{k+1}}^2\right]
         &\le (1-\eta_k\gamma_c)\ex\left[\norm{e_k}^2\right]+\frac{2n^2\eta_k^2\cj^2}{\delta^2}+\frac{\alpha\eta_k\delta^2}{\gamma_c}.
\end{align}
A standard stochastic approximation comparison \cite{polyak} gives:
\begin{equation}\label{limxcbd}
    \limsup_{k\to\infty}\ex\left[\norm{x_{k}-x^*_{c,\delta}}^2\right]\le \frac{\alpha\delta^2}{\gamma_c^2}.
\end{equation}
We combine \eqref{limex}, \eqref{limxcbd}, and \eqref{dxcbd} from Lemma \ref{xcdbd} in the Appendix to prove \eqref{limxk}, with $R=\frac{2\alpha}{\gamma_c^2}+\frac{2}{r_c^2}(1+\frac{L_c}{\gamma_c})^2\norm{x_c^*-p}^2$. To establish \eqref{limxkk}, we apply Lemma \ref{lem4} in the Appendix to \eqref{exk1bd} and obtain $C>0$ such that the following estimate holds for all $k\in\N$ and $\delta\in (0, r_c)$:
\begin{equation}\label{exxkkk}
    \ex\left[\norm{x_k-x_{c,\delta}^*}^2\right]\le \frac{2n^2C\cj^2}{\gamma_c^2\delta^2 k}+\frac{\alpha\delta^2}{\gamma_c^2}.
\end{equation}
We substitute \eqref{exxkkk} into \eqref{limex} and apply Lemma \ref{xcdbd} from the Appendix to \eqref{limex} and obtain \eqref{limxkk}, with $C=\max\left\{\frac{q^2}{q-1}, \frac{\gamma_c^2 r_c^2\ceil{q}\max_{x,y\in\cx}\norm{x-y}^2}{2n^2\cj^2}\right\}$, $M=\frac{4Cn^2\cj^2}{\gamma_c^2}$.\pQEDB

\subsection{Proof of Theorem \ref{thm_inj}}
We start by proving the first part. Since $x^*\in\text{int}~\cx$, we know that $\cg(x^*)=0$. Let $\ep=\frac{\gamma}{L_{12} \norm{\Gamma}}$, and we define the mapping $g:(-\ep, \ep)\to\cx$ given by $g(c)=x_c^*$. Since $x^*\in\text{int}~\cx$ and \eqref{xcxbd} holds for $c\in (-\ep,\ep)$, there exists a neighborhood $\ce\subset (-\ep, \ep)$ of 0 such that $g(\ce)\subset\text{int}~\cx$, and hence $\gdec(g(c))=0$ for all $c\in \ce$. If we consider some $c_1, c_2\in \ce$ with $c_1\neq c_2$, we have 
    \begin{equation}\label{ginj}
        \cg(g(c_i))+c_i\begin{bmatrix}
            \Gamma^\top\nabla_2 J_1(g(c_i))\\ 0
        \end{bmatrix}=0,
    \end{equation}
    for $i\in\{1,2\}$. If we suppose that $g(c_1)=g(c_2)=\hat{x}$, we consider the equation \eqref{ginj} for $i=1$, $i=2$ and subtract to obtain
    \begin{equation}\label{ginj2}
        (c_1-c_2)\begin{bmatrix}
            \Gamma^\top\nabla_2 J_1(\hat{x})\\ 0
        \end{bmatrix}=0.
    \end{equation}
    Since $c_1-c_2\neq 0$, the equation \eqref{ginj2} is only valid if $\Gamma^\top\nabla_2 J_1(\hat{x})=0$. This would imply that $\cg(\hat{x})=0$, so by the uniqueness of the NE, we conclude that $\hat{x}=x^*$. However, we also assumed that $\Gamma^\top\nabla_2 J_1(x^*)\neq 0$, so we arrive at a contradiction. Hence, $g$ is injective on $\ce$, which proves the first part. For the second part,
$\Gamma^\top\nabla_2J_1(x^\ast)=0$ implies
$\gdec(x^\ast)=\cg(x^\ast)$.
Since $x^\ast$ solves \eqref{varineq}, it also solves \eqref{videc}. Strong monotonicity of
$\gdec$ then yields $x_c^\ast=x^\ast$.\pQEDB

\subsection{Proof of Theorem \ref{thm_imp}}
Since $\cg$ is strongly monotone, the implicit function theorem yields $\left.\frac{\partial x_c^*}{\partial c}\right|_{c=0}=-\nabla\mathcal{G}(x^*)^{-\top} \begin{bmatrix} \Gamma^\top \nabla_2 J_1(x^*) \\ 0 \end{bmatrix}.$ Since $x^*\in\text{int} ~\cx$, also have $\nabla_2 J_2(x^*)=0$.
    A straightforward application of chain rule reveals that the left hand side of \eqref{implicit} precisely equals $-\frac{\partial J_2(x_c^*)}{\partial c}|_{c=0}$. If $\frac{\partial J_2(x_c^*)}{\partial c}|_{c=0}>0$, we can find $\widetilde{\ep}\in (0,\ep)$ such that the result holds for $\ce=(-\widetilde{\ep}, 0)$, \tcb{where $\ep$ is obtained from the proof of Theorem \ref{thm_inj}.} Otherwise, we will have $\ce=(0, \widetilde{\ep})$.\pQEDB
\section{Numerical Example}
\label{sec_num}

To illustrate our results, we apply them to a stylized two-player resource allocation problem. Consider two players allocating demand across two shared resources. Player $i$ has action $x_i\in\re^2$, where the $j$-th component of their action specifies the amount of demand assigned to resource $j$. The cost function of player $i$ is expressed in the form
\begin{equation*}
    J_i(x)=\norm{x_i}^2+x_i^\top x_{3-i}+q_i^\top x_i+\lambda\sum_{j=1}^2 [x_{ij}-\tau_{ij}]_+^\frac52,
\end{equation*}
where $[\cdot]_+:=\max\{0,\cdot\}$ and the first three terms model private allocation costs, cross-player congestion, and linear allocation value, respectively. The final term is a soft capacity penalty, where $\tau_{ij}$ is player $i$'s reference allocation for resource $j$. Here, we will use $\lambda=0.12$, $q_1=(-4,-5), q_2=(-6,-4), \tau_1=(\frac23,2)$, $\tau_2=(\frac83,1)$, and $\cx_1=\cx_2=[0, 5]^2$. This example satisfies Assumption \ref{assumps} and falls outside the class of second-order smoothness conditions required in \cite{9992950}.

\tcb{To seek the Nash equilibrium $x^*=((2/3,2),(8/3,1))$, each player uses a learning algorithm based on bandit feedback. Player 1 implements \eqref{dx1}, while, at each stage, player 2 obtains noisy measurements of $\mu_{1,k}$ with probability $0.75$, which are stored in a signal $\Theta_{1,k}\in\re^2$ that assumes the form given in \eqref{th1k}. We have $\Lambda_k=b_k(\id_2+\Xi_k)$ where $b_k\sim\text{Bernoulli}(0.75)$ and, for each $k\in\N$, each entry of the random matrix $\Xi_k$ is sampled independently and uniformly on $[-0.15, 0.15]$. The additive noise satisfies $w_k=0.1 b_k z_k$, where the components of $z_k$ are independent standard Gaussian random variables truncated to $[-3,3]$. Player 2 implements the dynamics \eqref{dx2} with $c=0.8$, and each player uses $\delta=0.25$ with step sizes $\eta_k={3}{k^{-1}}$. It is verifiable that $\gdec$ is strongly monotone and the DNE is given by $(x_c^*)_1\approx(0.48,1.44)$ and $(x_c^*)_2\approx(2.76, 1.26)$. Since $J_2(x^*)-J_2(x_c^*)\approx 1.15$, player 2 improves their steady state cost through deception. The averaged trajectories of $J_i$ and $x$ taken over 300 realizations are given in Figure \ref{figs}, which also illustrates how the parameter $c$ affects the steady-state costs and equilibrium shift.} Figure \ref{counter} shows that countermeasures that suppress the same-round dither-signal correlation, which include an induced dither-signal delay, zero-mean masking, and dither resetting, recover the nominal learning trajectories and converge to a neighborhood of the NE, while the undefended dynamics remain displaced.

\begin{figure}[t!]
  \centering
    \vspace{0.2cm}\includegraphics[width=0.43\textwidth]{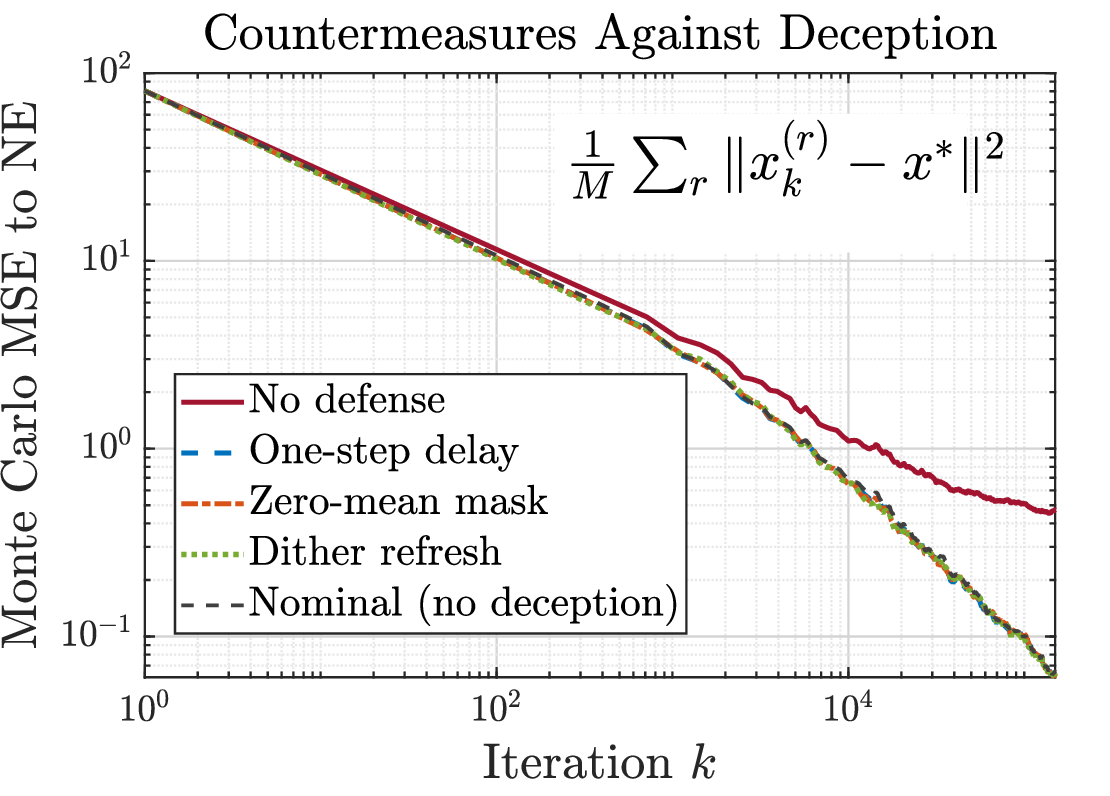}
    \vspace{-0.4cm}
    \caption{\small Convergence of the action profile with deception countermeasures implemented at each stage. One-step delay: player 2 uses $\Theta_{1,k-1}$ during stage $k$. Zero-mean mask: player 1 privately flips the sign of the dither with probability $\frac12$ after leakage. Dither refresh: player 1 resamples the dither after leakage.}\label{counter}
    \vspace{-0.4cm}
\end{figure}
\section{Conclusion}
\label{sec_conc}
We introduced a novel framework for deceptive bandit learning in strongly monotone games, demonstrating that noisy information correlated with an opponent’s exploration can induce an externality that steers the joint action profile toward a deceptive Nash equilibrium. We then proved that the DBL dynamics converge to a $\co(\delta)$ RMS neighborhood of the DNE, establishing optimal rates under weaker regularity conditions than \cite{9992950} and revealing a novel relaxation that also applies in the deception-free setting. Furthermore, we established geometric conditions under which the deception mechanism is guaranteed to shift the steady state and improve the deceiver's long-term cost/payoff. Our results were illustrated in a two-player resource allocation setting. Future directions include extending the ideas to asynchronous games, studying the effect of the network topology, and investigating active countermeasures against deception. 

\bibliographystyle{IEEEtran}
\bibliography{Biblio.bib}

\begin{thebibliography}{10}
\providecommand{\url}[1]{#1}
\csname url@samestyle\endcsname
\providecommand{\newblock}{\relax}
\providecommand{\bibinfo}[2]{#2}
\providecommand{\BIBentrySTDinterwordspacing}{\spaceskip=0pt\relax}
\providecommand{\BIBentryALTinterwordstretchfactor}{4}
\providecommand{\BIBentryALTinterwordspacing}{\spaceskip=\fontdimen2\font plus
\BIBentryALTinterwordstretchfactor\fontdimen3\font minus \fontdimen4\font\relax}
\providecommand{\BIBforeignlanguage}[2]{{%
\expandafter\ifx\csname l@#1\endcsname\relax
\typeout{** WARNING: IEEEtran.bst: No hyphenation pattern has been}%
\typeout{** loaded for the language `#1'. Using the pattern for}%
\typeout{** the default language instead.}%
\else
\language=\csname l@#1\endcsname
\fi
#2}}
\providecommand{\BIBdecl}{\relax}
\BIBdecl

\bibitem{flaxman2005online}
A.~Flaxman, A.~Kalai, and H.~McMahan, ``Online convex optimization in the bandit setting: Gradient descent without a gradient,'' in \emph{Proceedings of the Annual ACM-SIAM Symposium on Discrete Algorithms}, 2005, pp. 385--394.

\bibitem{landgren2021distributed}
P.~Landgren, V.~Srivastava, and N.~E. Leonard, ``Distributed cooperative decision making in multi-agent multi-armed bandits,'' \emph{Automatica}, vol. 125, p. 109445, 2021.

\bibitem{liu2010distributed}
K.~Liu and Q.~Zhao, ``Distributed learning in multi-armed bandit with multiple players,'' \emph{IEEE transactions on signal processing}, vol.~58, no.~11, pp. 5667--5681, 2010.

\bibitem{shahrampour2017multi}
S.~Shahrampour, A.~Rakhlin, and A.~Jadbabaie, ``Multi-armed bandits in multi-agent networks,'' in \emph{2017 IEEE International Conference on Acoustics, Speech and Signal Processing (ICASSP)}.\hskip 1em plus 0.5em minus 0.4em\relax IEEE, 2017, pp. 2786--2790.

\bibitem{zhu2023distributed}
J.~Zhu and J.~Liu, ``Distributed multiarmed bandits,'' \emph{IEEE Transactions on Automatic Control}, vol.~68, no.~5, pp. 3025--3040, 2023.

\bibitem{sankararaman2019social}
A.~Sankararaman, A.~Ganesh, and S.~Shakkottai, ``Social learning in multi agent multi armed bandits,'' \emph{Proceedings of the ACM on Measurement and Analysis of Computing Systems}, vol.~3, no.~3, pp. 1--35, 2019.

\bibitem{misra2019dynamic}
K.~Misra, E.~M. Schwartz, and J.~Abernethy, ``Dynamic online pricing with incomplete information using multiarmed bandit experiments,'' \emph{Marketing Science}, vol.~38, no.~2, pp. 226--252, 2019.

\bibitem{mueller2019low}
J.~W. Mueller, V.~Syrgkanis, and M.~Taddy, ``Low-rank bandit methods for high-dimensional dynamic pricing,'' \emph{Advances in Neural Information Processing Systems}, vol.~32, 2019.

\bibitem{119632}
J.~Spall, ``Multivariate stochastic approximation using a simultaneous perturbation gradient approximation,'' \emph{IEEE Transactions on Automatic Control}, vol.~37, no.~3, pp. 332--341, 1992.

\bibitem{littman1994markov}
M.~L. Littman, ``Markov games as a framework for multi-agent reinforcement learning,'' in \emph{Machine learning proceedings 1994}.\hskip 1em plus 0.5em minus 0.4em\relax Elsevier, 1994, pp. 157--163.

\bibitem{malik2019derivative}
D.~Malik \emph{et~al.}, ``Derivative-free methods for policy optimization: Guarantees for linear quadratic systems,'' in \emph{The 22nd international conference on artificial intelligence and statistics}.\hskip 1em plus 0.5em minus 0.4em\relax PMLR, 2019, pp. 2916--2925.

\bibitem{marden2009payoff}
J.~R. Marden, H.~P. Young, G.~Arslan, and J.~S. Shamma, ``Payoff-based dynamics for multiplayer weakly acyclic games,'' \emph{SIAM Journal on Control and Optimization}, vol.~48, no.~1, pp. 373--396, 2009.

\bibitem{tatarenko2018learning}
T.~Tatarenko and M.~Kamgarpour, ``Learning generalized {N}ash equilibria in a class of convex games,'' \emph{IEEE Transactions on Automatic Control}, vol.~64, no.~4, pp. 1426--1439, 2018.

\bibitem{pang2020distributed}
Y.~Pang and G.~Hu, ``Distributed {N}ash equilibrium seeking with limited cost function knowledge via a consensus-based gradient-free method,'' \emph{IEEE Transactions on Automatic Control}, vol.~66, no.~4, pp. 1832--1839, 2020.

\bibitem{vamvoudakis2025deception}
K.~G. Vamvoudakis \emph{et~al.}, ``Deception in game theory and control: A tutorial,'' in \emph{2025 American Control Conference (ACC)}.\hskip 1em plus 0.5em minus 0.4em\relax IEEE, 2025, pp. 2146--2163.

\bibitem{ye2021differentially}
M.~Ye, G.~Hu, L.~Xie, and S.~Xu, ``Differentially private distributed {N}ash equilibrium seeking for aggregative games,'' \emph{IEEE Transactions on Automatic Control}, vol.~67, no.~5, pp. 2451--2458, 2021.

\bibitem{pmlr-v125-boursier20a}
E.~Boursier and V.~Perchet, ``Selfish robustness and equilibria in multi-player bandits,'' in \emph{Conference on Learning Theory}.\hskip 1em plus 0.5em minus 0.4em\relax PMLR, 2020, pp. 530--581.

\bibitem{9992950}
D.~Drusvyatskiy, M.~Fazel, and L.~J. Ratliff, ``Improved rates for derivative free gradient play in strongly monotone games,'' in \emph{2022 IEEE 61st Conference on Decision and Control (CDC)}, 2022, pp. 3403--3408.

\bibitem{yao2025single}
F.~Yao, Y.~Cheng, E.~Wei, and H.~Xu, ``Single-agent poisoning attacks suffice to ruin multi-agent learning,'' in \emph{The Thirteenth International Conference on Learning Representations}, 2025.

\bibitem{wu2023reward}
Y.~Wu, J.~McMahan, X.~Zhu, and Q.~Xie, ``Reward poisoning attacks on offline multi-agent reinforcement learning,'' in \emph{Proceedings of the {AAAI} conference on artificial intelligence}, vol.~37, no.~9, 2023, pp. 10\,426--10\,434.

\bibitem{gleave2019adversarial}
A.~Gleave, M.~Dennis, C.~Wild, N.~Kant, S.~Levine, and S.~Russell, ``Adversarial policies: Attacking deep reinforcement learning,'' \emph{arXiv preprint arXiv:1905.10615}, 2019.

\bibitem{11048686}
M.~Tang, U.~Javed, X.~Chen, M.~Krstić, and J.~I. Poveda, ``Deception in {N}ash equilibrium seeking,'' \emph{IEEE Transactions on Automatic Control}, vol.~70, no.~12, pp. 7984--7999, 2025.

\bibitem{tang2025stochastic}
M.~Tang, M.~Krstic, and J.~Poveda, ``Stochastic real-time deception in {N}ash equilibrium seeking for games with quadratic payoffs,'' in \emph{7th Annual Learning for Dynamics \& Control Conference}.\hskip 1em plus 0.5em minus 0.4em\relax PMLR, 2025, pp. 635--646.

\bibitem{yan2023fast}
Y.-H. Yan, P.~Zhao, and Z.-H. Zhou, ``Fast rates in time-varying strongly monotone games,'' in \emph{International Conference on Machine Learning}.\hskip 1em plus 0.5em minus 0.4em\relax PMLR, 2023, pp. 39\,138--39\,164.

\bibitem{bravo2018bandit}
M.~Bravo, D.~Leslie, and P.~Mertikopoulos, ``Bandit learning in concave n-person games,'' \emph{Advances in Neural Information Processing Systems}, vol.~31, 2018.

\bibitem{ba2025doubly}
W.~Ba, T.~Lin, J.~Zhang, and Z.~Zhou, ``Doubly optimal no-regret online learning in strongly monotone games with bandit feedback,'' \emph{Operations Research}, vol.~73, no.~6, pp. 3219--3244, 2025.

\bibitem{jordan2025adaptive}
M.~Jordan, T.~Lin, and Z.~Zhou, ``Adaptive, doubly optimal no-regret learning in strongly monotone and exp-concave games with gradient feedback,'' \emph{Operations Research}, vol.~73, no.~3, pp. 1675--1702, 2025.

\bibitem{bergemann2016bayes}
D.~Bergemann and S.~Morris, ``Bayes correlated equilibrium and the comparison of information structures in games,'' \emph{Theoretical Economics}, vol.~11, no.~2, pp. 487--522, 2016.

\bibitem{solan2004games}
E.~Solan and L.~Yariv, ``Games with espionage,'' \emph{Games and Economic Behavior}, vol.~47, no.~1, pp. 172--199, 2004.

\bibitem{alon2013adversarial}
N.~Alon, Y.~Emek, M.~Feldman, and M.~Tennenholtz, ``Adversarial leakage in games,'' \emph{SIAM Journal on Discrete Mathematics}, vol.~27, no.~1, pp. 363--385, 2013.

\bibitem{alvim2017information}
M.~S. Alvim, K.~Chatzikokolakis, Y.~Kawamoto, and C.~Palamidessi, ``Information leakage games,'' in \emph{International Conference on Decision and Game Theory for Security}.\hskip 1em plus 0.5em minus 0.4em\relax Springer, 2017, pp. 437--457.

\bibitem{6060862}
P.~Frihauf, M.~Krstic, and T.~Basar, ``{N}ash equilibrium seeking in noncooperative games,'' \emph{IEEE Transactions on Automatic Control}, vol.~57, no.~5, pp. 1192--1207, 2012.

\bibitem{9992906}
B.~Gao and L.~Pavel, ``Bandit learning with regularized second-order mirror descent,'' in \emph{2022 IEEE 61st Conference on Decision and Control (CDC)}, 2022, pp. 5731--5738.

\bibitem{perolat2021poincare}
J.~Perolat \emph{et~al.}, ``From poincar{\'e} recurrence to convergence in imperfect information games: Finding equilibrium via regularization,'' in \emph{ICML}.\hskip 1em plus 0.5em minus 0.4em\relax PMLR, 2021, pp. 8525--8535.

\bibitem{polyak}
B.~Polyak, \emph{Introduction to Optimization}, ser. Translations series in mathematics and engineering.\hskip 1em plus 0.5em minus 0.4em\relax Optimization Software, Publications Division, 1987.

\end{thebibliography}

\section{Appendix}
\label{sec_ap}
Here we present several technical lemmas that are used to prove our main results.
\vspace{-5pt}
\begin{lemma}\label{xcdbd}
    Suppose $\gdec$ is $\gamma_c$-strongly monotone, and let $x^*_{c,\delta}$ denote a Nash equilibrium of $(\gdec, \cx_1^\delta\times\cx_2^{\delta_c})$, where $\delta$ satisfies \eqref{delta_bd}. Then, the following holds
    \begin{equation}\label{dxcbd}
        \norm{x^*_c-x^*_{c,\delta}}\le \tcb{d_c}\left(1+\frac{L_c}{\gamma_c}\right) \norm{x_c^*-p},
    \end{equation}
    \tcb{where $d_c=\max\left(\frac{\delta}{r_1}, \frac{\delta_c}{r_2}\right)$, $L_c:=|c|\norm{\Gamma}L_{12}+\sqrt{L_{11}^2+L_{22}^2}$, and $p=(p_1,p_2)$.}
\end{lemma}

\vspace{-12pt}
\begin{proof}
    Define $\cx^\delta_c:=\cx_1^\delta\times\cx_2^{\delta_c}$, which is convex and compact. Letting $p=(p_1, p_2)$, we define the following matrix and transformation:
    \begin{equation*}
        \mathcal{D}:=\begin{bmatrix}
            \left(1-\frac{\delta}{r_1}\right) \id_n & 0\\ 0 & \left(1-\frac{\delta_c}{r_2}\right)\id_n
        \end{bmatrix},~~~ T(x)=\cd(x-p)+p.
    \end{equation*}
    Note that $\cx_c^\delta=T(\cx)$. By the strong monotonicity of $\gdec$, it also follows that the mapping $x\mapsto \gdec(x)+\cn_{\cx^\delta_c}(x)$ is also $\gamma_c$-strongly monotone, where $\cn_{\cx^\delta_c}(x)$ is the normal cone of $\cx^\delta_c$ at the point $x$. We then obtain
    \begin{equation*}
        \gamma_c\norm{T(x_c^*)-x^*_{c,\delta}}\le \inf_{v\in \gdec(T(x_c^*))+\cn_{\cx^\delta_c}(T(x_c^*))}\norm{v}.
    \end{equation*}
    Since $x^*_c$ is a DNE, we have that $0\in \gdec( x^*_c)+\cn_{\cx}(x^*_c).$
    Using the fact that $\cn_{\cx^\delta_c}(T(x_c^*))=\cn_{\cx}(x_c^*)$, we obtain
    \begin{align*}
        \inf_{v\in \gdec(T(x_c^*))+\cn_{\cx^\delta_c}(T(x_c^*))}\norm{v}&=\inf_{v\in \gdec(T(x_c^*))+\cn_{\cx}( x^*_c)}\norm{v}\\
        &\le \norm{\gdec( T(x_c^*))-\gdec(x_c^*)}\\
        &\le L_c \norm{(\cd-\id_{2n})(x_c^*-p)}\\
        &\le L_c \tcb{d_c}\norm{x_c^*-p},
    \end{align*}
    where $L_c$ is a Lipschitz constant of $\gdec$.
    By combining the above computations, we obtain
    \begin{align*}
        \norm{x^*_c-x^*_{c,\delta}}&\le \norm{x_c^*- T(x_c^*)}+\norm{T(x_c^*)-x^*_{c,\delta}}\\
        &\le \tcb{d_c} \left(1+\frac{L_c}{\gamma_c}\right) \norm{x_c^*-p},
    \end{align*}
    as desired.
\end{proof}
\vspace{-10pt}
\begin{lemma}\label{lem_o}
    In \eqref{expg}, for each $k\in\mathbb{N}$, the following estimates hold almost surely:
    \begin{subequations}\label{bias_bd}
        \begin{align}\label{ordbd}
            \norm{\beta_{1,k}^\delta}&\le  \delta(L_{11}+|c|\bar{\Lambda} L_{12})(2+|c|(\bar{\Lambda}+\bar{w}))\\
            \norm{\beta_{2,k}^\delta}&\le \delta L_{22}(2+|c|(\bar{\Lambda}+\bar{w})).
        \end{align}
    \end{subequations}   
\end{lemma}
\begin{proof}
    Given $k\in\N$, we recall that
    \begin{equation}\label{xbar}
        \bar{x}_k=\begin{bmatrix}
            x_{1,k}+\delta\mu_{1,k} \\ x_{2,k}+\delta\mu_{2,k}+c\delta(\Lambda_k\mu_{1,k}+w_k)
        \end{bmatrix}.
    \end{equation}
    Since $k$ will remain fixed, we drop it from the subscripts for notational convenience. By Stokes' theorem:
    \begin{subequations}
        \begin{align*}
            \ex_k \left[\frac{n}{\delta}J_1(\bar{x})\mu_{1}\right]&=\ex_k^{\mu_1\sim B,\mu_2\sim S}\left[\nabla_1 J_1(\bar{x})+c\Lambda^\top\nabla_2 J_1(\bar{x})\right]\\
            \ex_k \left[\frac{n}{\delta}J_2(\bar{x})\mu_{2}\right]&=\ex_k^{\mu_1\sim S,\mu_2\sim B}\left[\nabla_{2}J_2(\bar{x})\right],
        \end{align*}
    \end{subequations}
    where $S$ and $B$ are uniform distributions over $\bs^n$ and $\B^n$, respectively. Since $\ex_k[\Lambda]=\Gamma$, we can continue from \eqref{expg} to obtain 
    \begin{align*}
        \norm{\beta_{1}^\delta}&\le\big\Vert\ex_k^{\mu_1\sim B,\mu_2\sim S}\left[\nabla_1 J_1(\bar{x})+c\Lambda^\top\nabla_2 J_1(\bar{x})\right]\\&~~~-\nabla_1 J_1(x)-c\Gamma^\top \nabla_2 J_1(x)\big\Vert\\
        &\le \ex_k^{\mu_1\sim B,\mu_2\sim S}\big[\norm{\nabla_1 J_1(\bar{x})-\nabla_1 J_1(x)}\\&~~~+\norm{c\Lambda^\top (\nabla_2 J_1(\bar{x})-\nabla_2 J_1(x))}\big]\\
        &\le \ex_k^{\mu_1\sim B,\mu_2\sim S}\left[\left(L_{11}+|c|\norm{\Lambda}L_{12}\right)\norm{\bar{x}-x}\right],
    \end{align*}
    and
    \begin{align*}
        \norm{\beta_{2}^\delta}&\le \ex_k^{\mu_1\sim S,\mu_2\sim B}\left[\norm{\nabla_2 J_2(\bar{x})-\nabla_2 J_2(x)}\right]\\
        &\le \ex_k^{\mu_1\sim S,\mu_2\sim B}\left[L_{22}\norm{\bar{x}-x}\right].
    \end{align*}
    The result is obtained by substituting \eqref{xbar} into the above expressions and using the facts that $\norm{\mu_{i}}\le 1$ and $\norm{\Lambda}\le\bar{\Lambda}$ almost surely.
\end{proof}
\vspace{-9pt}
\begin{lemma}\label{lem4}
    Consider a nonnegative sequence $\{V_k\}_{k=1}^\infty$ satisfying $V_{k+1}\le (1-a_k)V_k+A a_k^2+B a_k,$
    where $a_k=\frac{q}{k}$ for some $q>1$, and $A,B>0$. Then, there exists $C>0$ such that $V_k\le \frac{CA}{k}+B$ for all $k\in\N$.
\end{lemma}
\begin{proof}
    The desired bound clearly holds for $k\le\ceil{q}$, since we can always choose $C\ge \frac{1}{A}\max_{k\le \ceil{q}}kV_k$, so suppose it holds for some $k\ge\ceil{q}$. By the inductive hypothesis, we have $V_{k+1}\le\frac{CA(k-q)+A q^2}{k^2}+B$. It suffices to show $\frac{CA(k-q)+A q^2}{k^2}\le \frac{CA}{k+1},$
    which can be equivalently rearranged into $k\left(CA(1-q)+A q^2\right)-CA q+A q^2\le 0.$
    Direct computation shows that this holds for all $k\in\N$ when $C\ge \frac{ q^2}{q-1}$. We set $C\ge\max\left\{\frac{ q^2}{q-1}, \frac{1}{A}\max_{k\le \ceil{q}}kV_k\right\}$ to conclude.
\end{proof}

\end{document}